\documentclass{article}

\usepackage{natbib}
\usepackage{graphicx} 
\usepackage{tikz}
\usepackage{booktabs}
\usepackage[left=1.5cm, right=2cm]{geometry}
\newgeometry{left=3cm,bottom=4cm}
\usepackage{algorithm,algpseudocode}

\usepackage{subcaption}
\usepackage{tikz}
\usetikzlibrary{math}
\usetikzlibrary{3d}
\usetikzlibrary{backgrounds}
\usetikzlibrary{calc}
\usetikzlibrary{arrows}
\usetikzlibrary{shapes.geometric,arrows.meta, angles}
\usetikzlibrary{angles,quotes} % for right angle marks% Outer rectangle: angle at A

\usepackage{tkz-euclide}
\usepackage{xcolor}
\usepackage{amsmath,amsthm, amssymb,amsfonts}
\usepackage{times}
\usepackage{latexsym}
\usepackage{verbatim}

\usepackage{pgfplots}
\pgfdeclarelayer{bg}
\pgfsetlayers{bg,main}

\newtheorem{theorem}{Theorem}

\newtheorem{lemma}[theorem]{Lemma}
\newtheorem{proposition}[theorem]{Proposition}

\newcommand{\R}{\mathbb{R}}

\def\bb0{{\mathbb{0}}}

\def\bb{{\mathbf{b}}}

\def\b0{{\mathbf{0}}}

\def\b1{{\mathbf{1}}}

\def\bbR{{\mathbb{R}}}

\def\cV{\mathcal{V}}

\def\cX{\mathcal{X}}

\def\sf0{{\mathsf{0}}}

\def\dist{\text{dist}}
\def\PP{\mathcal{P}}
\def\GG{\mathcal{G}}

\tikzset{every node/.style={font=\footnotesize}}

\usetikzlibrary{intersections} % put this in your preamble
\usetikzlibrary{calc,angles,quotes}

\def\RR{\mathbb{R}}

\def\PP{\mathcal{P}}

\def\BB{\mathcal{B}}

\def\dist{\textsf{dist}}
\def\AA{\mathcal{A}}

\pgfplotsset{compat=1.18}

\title{Lower Bound on the Cumulative Constraint Violation for the OGD+Projection algorithm for Constrained Online Convex Optimization (COCO)}

\author{Haricharan Balasundaram, Karthick Krishna Mahendran, Rahul Vaze}

\date{}

\begin{document}

\maketitle

\begin{abstract}
The problem of constrained online convex optimization is considered, where at each round, once a learner commits to an action $x_t \in \mathcal{X} \subset \mathbb{R}^d$, a convex loss function $f_t$ and a convex constraint function $g_t$ that drives the constraint $g_t(x)\le 0$ are revealed. The objective is to simultaneously minimize the static regret and cumulative constraint violation (CCV) compared to the benchmark that knows the loss functions and constraint functions $f_t$ and $g_t$ for all $t$ ahead of time, and chooses a static optimal action that is feasible with respect to all $g_t(x)\le 0$. Currently, the best known algorithm is OGD+Projection algorithm of \citep{vaze2025osqrttstaticregretinstance} that has  simultaneous regret of $O(\sqrt{T})$ and CCV of $O(T^{1/3})$ for $d=2$ \citep{balasundaram2026breakingosqrttcumulativeconstraint}, and simultaneous regret of $O(\sqrt{T})$  and CCV of $O(\sqrt{T})$ for any $d$ \citep{sarkar2026improvedguaranteesconstrainedonline}. In this paper, we show that the CCV of the OGD+Projection algorithm is $\Omega (T^{\frac{d-1}{2d}})$. This is the first such lower bound result. 

\end{abstract}
\section{Introduction}

We consider constrained online convex optimization (COCO), where on every round $t$, an online algorithm first chooses an admissible action $x_t \in \mathcal{X} \subset \bbR^d$, and then the adversary chooses a convex loss/cost function $f_t: \mathcal{X} \to \mathbb{R}$ and a constraint function of the form $g_{t}(x) \leq 0,$ where $g_{t}: \mathcal{X} \to \mathbb{R}$ is a convex function. Let $\mathcal{X}^\star$ be the feasible set consisting of all admissible actions that satisfy all constraints $g_{t}(x) \leq 0, t\in [T]$. We work under the standard assumption that $\mathcal{X}^\star$ is not empty (called the {\it feasibility assumption}). 

Since $g_{t}$'s are revealed after the action $x_t$ is chosen, an online algorithm need not necessarily take feasible actions on each round, and in addition to the static regret 
\begin{equation} \label{eqn:intro-regret-def}
	\textrm{Regret}_{[1:T]} \equiv \sup_{\{f_t\}_{t=1}^T} \sup_{x^\star \in \mathcal{X^\star}} \textrm{Regret}_T(x^\star),
\end{equation}
where $\textrm{Regret}_T(x^\star) \equiv \sum_{t=1}^T f_t(x_t) - \sum_{t=1}^T f_t(x^\star)$,
an additional metric of interest is the cumulative constraint violation (CCV) defined as 
\begin{equation} \label{eqn:intro-gen-oco-goal}
 	\textrm{CCV}_{[1:T]}  \equiv \sum_{t=1}^T \max(g_{t}(x_t),0) = \sum_{t=1}^T (g_{t}(x_t))^+. 
\end{equation}
The goal is to design an online algorithm to simultaneously achieve a small regret \eqref{eqn:intro-regret-def} with respect to any admissible benchmark $x^\star \in \mathcal{X}^\star$ and a small CCV \eqref{eqn:intro-gen-oco-goal}. 

With constraint sets ${\cal G}_t = \{x\in \cX : g_t(x)\le 0\}$ being convex for all $t$, and the assumption $\mathcal{X}^\star = \cap_t \GG_t  \neq \varnothing $,  note that the sets $S_t = \cap_{\tau=1}^t {\cal G}_\tau$ are convex and are nested, i.e. $S_t\subseteq S_{t-1}$ and $x^\star \in S_t$ for all $t$. Essentially, set $S_t$'s are sufficient to quantify the CCV.

\subsection{Prior Work}

\textbf{Constrained OCO (COCO): (A) Time-invariant constraints:} COCO with time-invariant constraints, \emph{i.e.,} $g_{t} = g, \forall \ t$ \cite{yuan2018online, jenatton2016adaptive, mahdavi2012trading, yi2021regret} has been considered extensively, where the function $g$ is assumed to be known to the algorithm \emph{a priori}. The algorithm is allowed to take actions that are infeasible at any time to avoid the costly projection step of the vanilla projected OGD algorithm and the main objective was to design an \emph{efficient} algorithm  with a small regret and CCV while avoiding  the explicit projection step. 

\textbf{(B) Time-varying constraints:} The more difficult question is solving the COCO problem when the constraint functions, \emph{i.e.} $g_{t}$'s, change arbitrarily with time $t$. In this setting, all prior work on COCO made the feasibility assumption. One popular algorithm for solving COCO considered a Lagrangian function optimization that is updated using the primal and dual variables \cite{yu2017online, pmlr-v70-sun17a, yi2023distributed}. Alternatively, \cite{neely2017online} and \cite{georgios-cautious} used the drift-plus-penalty (DPP) framework  \cite{neely2010stochastic} to solve the COCO, but which needed additional assumptions, e.g. the Slater's condition in \cite{neely2017online} and with weaker form of the feasibility assumption \cite{neely2017online}'s. \cite{guo2022online} obtained the bounds similar to \cite{neely2017online} but without assuming Slater's condition. However, the algorithm \cite{guo2022online} was quite computationally intensive since it requires solving a convex optimization problem on each round. 

 Simultaneous bounds on regret $O (\sqrt{T})$ and CCV $O (\sqrt{T}\log T)$ for COCO were derived in \cite{Sinha2024} with a very simple algorithm that combines the loss function at time $t$ and the CCV accrued till time $t$ in a single loss function, and then executes the online gradient descent (OGD) algorithm on the single loss function with an adaptive step-size. Moreover, the result of \cite{Sinha2024} was shown to be tight in \cite[Lemma 6]{vaze2025osqrttstaticregretinstance} for an explicit input construction for which the algorithm of \cite{Sinha2024} has CCV of $\Omega(\sqrt{T}\log T)$ and that too for $d=1$. This was a consequence of the algorithm in \cite{Sinha2024} disregarding the geometry of the nested sets $S_t$'s and attempting to minimize both the regret and CCV for the worst-case input.

A geometry-aware algorithm was proposed in \cite{vaze2025osqrttstaticregretinstance} that first takes an OGD step with respect to the most recently revealed loss function $f_{t-1}$ and then projects that on to the most recently revealed constraint set $S_{t-1}$. For this algorithm, which we will refer to as {\it OGD+Projection}, an $O (\sqrt{T})$ regret bound and an instance specific CCV bound was established. In particular, the CCV was shown to be $O(1)$ when the sets are `nice' e.g., spheres or axis-aligned polygons, while in the general case, the CCV was shown to be $O(\cV)$, where $\cV$ is a parameter that depends on the distance between successive sets $S_t$'s and the shapes of sets $S_t$'s,  the dimension of the action space, and the diameter of the action space. Since no universal bound on $\cV$ was derived, CCV bound of $\min (\cV, O(\sqrt{T} \log T))$ was established by switching to the algorithm of  \cite{Sinha2024} in case $\cV$ exceeded $O(\sqrt{T})$. Thus, in the worst case, the bounds of \cite{Sinha2024} and \cite{vaze2025osqrttstaticregretinstance} are identical (regret of $O (\sqrt{T})$ and CCV of $O(\sqrt{T} \log T)$), however, for simple instances with $d=1$ for which the CCV bound of $O (\sqrt{T})$ \cite{Sinha2024} is tight, the CCV bound of \cite{vaze2025osqrttstaticregretinstance} is $O(1)$. 

More recently, \cite{balasundaram2026breakingosqrttcumulativeconstraint} showed that the OGD+Projection algorithm \cite{vaze2025osqrttstaticregretinstance} achieves CCV of $O(T^{1/3})$ for $d=2$, which was further extended in \cite{sarkar2026improvedguaranteesconstrainedonline} to show that CCV is  $O(\sqrt{T})$ universally independent of $d$. Thus, currently, the OGD+Projection algorithm is the best known algorithm for COCO. 
\subsection{Our Contributions}
To complement these upper bound results, in this paper for the first time we show that the CCV of OGD+Projection Algorithm is $\Omega (T^{\frac{d-1}{2d}})$. This shows that the upper bound analysis of OGD+Projection is not fundamentally loose.

\section{OGD+Projection Algorithm \citep{vaze2025osqrttstaticregretinstance}}
Let $\PP_{S} (x)$ be the projection of point $x$ on to $S$. 
 The OGD+Projection algorithm plays action $x_t$ at time $t$ for $1 \le t \le T$, where 
\begin{equation}
y_t = x_{t} - \eta_t \nabla f_t, \quad x_{t+1} = \PP_{S_t} (y_t)
\end{equation}
with $\eta_t = \frac{2 D}{G{\sqrt{t}}}$ in $d$ dimensions.  For simpler notation, we will refer to OGD+Projection algorithm as $\AA$. The main result of this paper is as follows.

\begin{theorem}\label{thm:main}
    There exists an instance of COCO of diameter $2 D$ and Lipschitzness $G$ such that the CCV of $\AA$ is $\Omega \left( T^{\frac{d-1}{2d}} \right)$.
\end{theorem}

To prove Theorem \ref{thm:main}, we will need the following construction summarized in Theorem \ref{thm:unit_vector_theorem}.

For \(d\ge 2\), define the set of unit vectors in $d$ dimensions,
\begin{equation}
\mathcal S^d=\{u\in\mathbb R^d:\|u\|=1\}.
\end{equation}
We equip \(\mathcal S^{d}\) with the angular separation,
\begin{equation}
\theta_{\mathcal S}(u,v)=\cos^{-1}(u^T v) \forall u, v \in \mathcal S^d.
\end{equation}

\begin{theorem}
\label{thm:unit_vector_theorem}
Let \(d\ge 2\). There exist constants \(A_d \ge 1\) and \(c_d>0\), depending only on \(d\), such that for every $\rho\in (0,\pi/2],$ there is an ordered collection of vectors
\begin{equation}
\mathcal U_d(\rho)=\left(u_1,\dots,u_{N_d(\rho)}\right)\subset \mathcal S^{d}
\end{equation}
with the following properties:

\begin{enumerate}
    \item For any two distinct indices \(i\neq j\), $\theta_{\mathcal S}(u_i,u_j)\ge \rho$.
    \item For every \(i=1,\dots,N_d(\rho)-1\), $
    \theta_{\mathcal S}(u_i,u_{i+1})\le A_d\rho$.
    \item The number of such vectors satisfies $N_d(\rho) \ge c_d\rho^{-(d-1)}$.
\end{enumerate}
\end{theorem}
The proof of Theorem \ref{thm:unit_vector_theorem} proceeds by induction and is relegated to Section~\ref{sec:unit_vector_proof}.
We next use Theorem~\ref{thm:unit_vector_theorem} for proving Theorem \ref{thm:main} as follows. Let \(A_d \ge 1\) and \(c_d>0\) be the constants from Theorem~\ref{thm:unit_vector_theorem}. 

\subsection{Construction of the vectors and constraint sets}

Take $\mathcal{X} = \BB(0, D)$, where $\BB(a, b)$ denotes a ball in the $\ell_2$ norm with center $a$ and radius $b$ in $d$ dimensions. This set has diameter $2D$. Let $M=T^{1/d}$ (we assume $T$ is such that $M$ is an integer) and define the radii
\begin{equation}
r_m=1-\frac{m-1}{2M},
\qquad
m=1,\dots,M+1.
\label{eqn:r_m_definition}
\end{equation}

Clearly, $
\frac 12 \le r_m \le 1 \, \, \forall m = 1, \dots, M+1. $
Let \(D\ge 1\) be a fixed constant depending only on \(d\), which will be chosen later. Define the concentric spheres with shrinking radii,
\begin{equation}
C_m=\partial \mathbb B(0,Dr_m),
\qquad
m=1,\dots,M+1.
\end{equation}

Suppose that we use Theorem~\ref{thm:unit_vector_theorem} to obtain unit vectors and scale the magnitude of each of the unit vectors by $r_m$ so that they lie on $C_m$. Let this ordered collection of vectors be $\chi_m$ (this will be defined precisely later). Let $x \in \chi_m$ and let $y = \PP_{C_{m+1}} (x)$ be the projection of this vector on to the next smaller concentric sphere. Our lower bound instance is such that the corresponding constraint set shown for this vector is a hyperplane which is tangent to $C_{m+1}$ at $y$ and contains the origin. The value of $\rho$ is selected so that this constraint does not remove any of the other vectors from $\chi_{m}$, as calculated below.

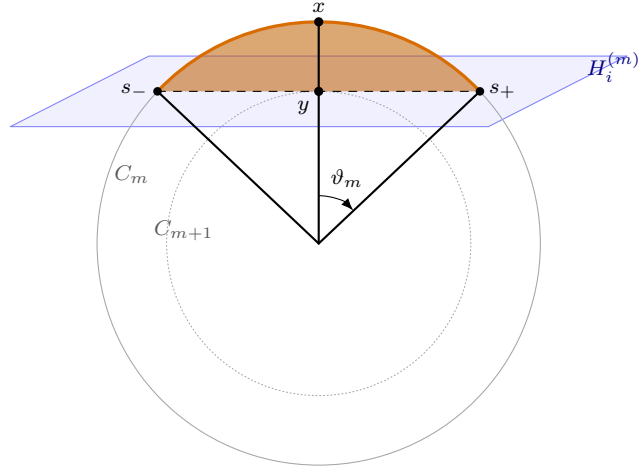
\begin{figure}[t]
\centering
\begin{tikzpicture}[
    scale=1.15,
    x={(1cm,0cm)},
    y={(0.55cm,0.28cm)},
    z={(0cm,1cm)},
    >=Latex,
    line cap=round,
    line join=round
]

% Illustrative radii for the picture only
\pgfmathsetmacro{\Rm}{2.55}
\pgfmathsetmacro{\Rp}{1.75}
\pgfmathsetmacro{\capr}{sqrt(\Rm*\Rm-\Rp*\Rp)}
\pgfmathsetmacro{\theta}{acos(\Rp/\Rm)}

\coordinate (O)       at (0,0,0);
\coordinate (xpt)     at (0,0,\Rm);
\coordinate (ypt)     at (0,0,\Rp);
\coordinate (splus)   at (\capr,0,\Rp);
\coordinate (sminus)  at (-\capr,0,\Rp);

% Tangent hyperplane H_i^{(m)}
\begin{scope}[canvas is xy plane at z=\Rp]
  \fill[blue!8,opacity=.75]
    (-2.75,-1.45) -- (2.75,-1.45) --
    (2.75,1.45) -- (-2.75,1.45) -- cycle;

  \draw[blue!50]
    (-2.75,-1.45) -- (2.75,-1.45) --
    (2.75,1.45) -- (-2.75,1.45) -- cycle;
\end{scope}

% Outer sphere C_m and cap cross-section
\begin{scope}[canvas is xz plane at y=0]
  \draw[gray!70] (0,0) circle[radius=\Rm];

  % Dark cap: region above the tangent plane, ending at s_- and s_+
  \fill[orange!75!black,opacity=.55]
    (90-\theta:\Rm)
    arc[start angle={90-\theta},end angle={90+\theta},radius=\Rm]
    -- cycle;

  \draw[orange!85!black,very thick]
    (90-\theta:\Rm)
    arc[start angle={90-\theta},end angle={90+\theta},radius=\Rm];
\end{scope}

% Inner sphere C_{m+1} cross-section
\begin{scope}[canvas is xz plane at y=0]
  \draw[gray!70,densely dotted] (0,0) circle[radius=\Rp];
\end{scope}

% Radius lines and tangent-plane chord in the displayed cross-section
\draw[thick] (O) -- (xpt);
\draw[thick] (O) -- (splus);
\draw[thick] (O) -- (sminus);
\draw[dashed] (sminus) -- (splus);

% Points
\fill (xpt) circle[radius=1.5pt]
  node[above] {$x$};

\fill (ypt) circle[radius=1.5pt]
  node[below left] {$y$};

\fill (splus) circle[radius=1.4pt]
  node[right] {$s_+$};

\fill (sminus) circle[radius=1.4pt]
  node[left] {$s_-$};

% Half-angle theta_m
\begin{scope}[canvas is xz plane at y=0]
  \draw[-Latex]
    (90:0.55)
    arc[start angle=90,end angle={90-\theta},radius=0.55];

  \node at ({90-\theta/2}:0.82) {$\vartheta_m$};
\end{scope}

% Minimal labels
\node[gray!70!black] at (-2.15,0,0.8) {$C_m$};
\node[gray!70!black] at (-1.55,0,0.15) {$C_{m+1}$};
\node[blue!50!black,anchor=west] at (2.0,1.0,\Rp) {$\ \ \ \ \ \ \ H_i^{(m)}$};

\end{tikzpicture}
\caption{The tangent hyperplane \(H_i^{(m)}\) to \(C_{m+1}\) at \(y\) cuts from \(C_m\) a spherical cap centered at \(x\). The points \(s_-\) and \(s_+\) are the two intersection points in the displayed cross-section, and the half-angle of the cap is \(\vartheta_m\).}
\label{fig:spherical-cap-half-angle}
\end{figure}

For some $x \in \chi_m$, we compute the half-angle of the spherical cap formed by a hyperplane which is tangent to \(C_{m+1}\) at $y$. For \(m=1,\dots,M\), let $
\vartheta_m = \cos^{-1}\left(\frac{r_{m+1}}{r_m}\right)$, as shown in Fig.~\ref{fig:spherical-cap-half-angle}. Note that,
\begin{equation}
1-\cos \vartheta_m = 1-\frac{r_{m+1}}{r_m} = \frac{1}{2M-m+1} \le \frac1M \forall m=1,\dots,M.
\end{equation}
Using $1-\cos x\ge \frac{x^2}{3}$ for $0\le x\le \frac{\pi}{2}$, we get
\begin{equation}
\vartheta_m\le \sqrt{\frac{3}{M}} = \sqrt{3}\,T^{-1/(2d)} \forall m=1,\dots,M. 
\end{equation}

We choose $\rho$ such that $\vartheta_m \le \rho$, so that removing the one vector $x$ from $\chi_m$ does not remove any other vectors from $\chi_m$. Take
\begin{equation}
\rho = \sqrt{\frac 3M} = \sqrt{3}\,T^{-1/(2d)}, \label{eqn:rho_definition}
\end{equation}
where $T$ is assumed to be large enough that $\rho \le \frac \pi 2$.

By Theorem~\ref{thm:unit_vector_theorem}, there exists an ordered collection $\mathcal U_d(\rho)=\left(v_1,\dots,v_{N_d(\rho)}\right)\subset \mathcal S^{d}$ satisfying all the properties in the theorem statement. For the choice of $\rho$ from~\eqref{eqn:rho_definition}, consider $n$ such that
\begin{equation}
N_d(\rho) \ge c_d\rho^{-(d-1)} = c_d3^{-(d-1)/2}T^{(d-1)/(2d)} \ge \tilde{c}_d T^{(d-1)/(2d)} = n.
\label{eqn:n_definition}
\end{equation}
for some $\tilde{c}_d \le c_d3^{-(d-1)/2}$ which will be chosen later. We assume $T$ is such that $n$ is an integer.

We use only the first \(n\) unit vectors $v_1,\dots,v_n$ hereon and ignore the other $N_d(\rho) - n$ unit vectors. These are `primary' unit vectors. These vectors will next be rotated and scaled by $r_m$ for $m = 1, \dots, M$ to obtain the $\chi_m$'s, i.e. the vectors we actually use in our construction.

\paragraph{Rotating the layers.}

We now rotate the copy of the ordered vector set from one layer to the next. This is to ensure that the last vector of the $m$th layer (i.e. $r_m v_n$) and the first vector of the $(m + 1)$th layer (i.e. $r_{m + 1} v_1$) are along the same direction, to ensure that the distance when moving from the last vector of one layer to the first vector of the subsequent layer is small. Since \(d\ge 2\), there exists an orthonormal rotation matrix $Q$ in $d$ dimensions such that
\begin{equation}
Qv_1 = v_n.
\end{equation}

For each layer \(m=1,\dots,M\), define
\begin{equation}
R_m=Q^{m-1}.
\label{eqn:rotation_def}
\end{equation}
Then $R_m$ is also a valid rotation matrix. Define the vectors on layer \(m\) by
\begin{equation}
u_i^{(m)}=R_m v_i,
\qquad
i=1,\dots,n.
\end{equation}
The set $\chi_m$ is the ordered collection of such vectors, i.e. 
\begin{equation}
    \chi_m = (u_1^{(m)}, u_2^{(m)}, \dots, u_n^{(m)}) \qquad m=1,\dots,M.
\end{equation}

Because rotations preserve angular distances,
\begin{equation}
\theta_{\mathcal S}\left(u_i^{(m)},u_j^{(m)}\right) \ge \rho \qquad \text{for all }i\neq j,
\end{equation}
\begin{equation}
\theta_{\mathcal S}\left(u_i^{(m)},u_{i+1}^{(m)}\right) \le A_d\rho \qquad \text{for }i=1,\dots,n-1.
\end{equation}

From~\eqref{eqn:rotation_def},
\begin{equation}
u_1^{(m+1)} = R_{m+1}v_1 = R_m v_n = u_n^{(m)}.
\end{equation}
Thus the last vector of layer \(m\) and the first vector of layer \(m+1\) are exactly aligned. This is the desired ``below each other'' property.

Finally, define the vectors
\begin{equation}
z_i^{(m)}=Dr_m u_i^{(m)} \in C_m, \qquad m=1 \dots,M, \qquad i=1,\dots,n.
\end{equation}
and let $Z$ be the ordered collection of vectors, defined as follows:
\begin{equation}
Z = (z_1^{(1)},\dots,z_n^{(1)},
z_1^{(2)},\dots,z_n^{(2)},
\dots,
z_1^{(M)},\dots,z_n^{(M)}).
\end{equation}

Next, we compute an upper bound for $\|z_p - z_{p+1}\|$ for $z_p, z_{p+1} \in Z$. Essentially, we attempt to compute the maximum distance between consecutive vectors in the ordered collection $Z$.

For consecutive vectors within the same layer,
\begin{align}
\left\|z_{i+1}^{(m)}-z_i^{(m)}\right\| &= Dr_m\left\|u_{i+1}^{(m)}-u_i^{(m)}\right\| \overset{(a)}{=} 2 D r_m \sin \left( \frac{\theta_{\mathcal S}\left(u_i^{(m)},u_{i+1}^{(m)}\right)}{2} \right) \\
&\overset{(b)}{\le} Dr_m \theta_{\mathcal S}\left(u_i^{(m)},u_{i+1}^{(m)}\right) \overset{(c)}{\le} DA_d\rho,
\label{eqn:consecutive_upper_bound}
\end{align}
where $(a)$ follows from trigonometric identities and using $\|u_{i}^{(m)}\| = \|u_{i+1}^{(m)}\| = 1$, $(b)$ uses $\sin x \le x$, and $(c)$ uses $r_m \le 1$.

For the transition from the last vector of layer \(m\) to the first vector of layer \(m+1\),
\begin{align}
\left\|z_1^{(m+1)}-z_n^{(m)}\right\|
&= \left\|Dr_{m+1}u_n^{(m)}-Dr_m u_n^{(m)}\right\| \\
&= D(r_m-r_{m+1}) = \frac{D}{2M}.
\label{eqn:consecutive_upper_bound_2}
\end{align}
where we use the definition of $r_m$ from~\eqref{eqn:r_m_definition}.

Since $\rho=\sqrt{\frac{3}{M}}$ from~\eqref{eqn:rho_definition}, we have $\frac{D}{2M} = D \frac{\rho^2}{6}$. Thus, combining~\eqref{eqn:consecutive_upper_bound} and~\eqref{eqn:consecutive_upper_bound_2}, we have
\begin{equation}
\|z_p - z_{p+1}\| \le \max\left(D A_d \rho, \frac{D \rho^2}{6} \right) = D A_d \rho \forall z_p, z_{p+1} \in Z,
\end{equation}
where we use $\frac 16 \rho \le A_d$ (recall that $A_d \ge 1$ and $\rho \le \frac{\pi}{2}$). This gives us the following proposition.
\begin{proposition}
\label{prop:small_z_distance}
For $z_{p}, z_{p+1} \in Z$, we have
\begin{equation}
\left\|z_{p+1}-z_p\right\| \le D A_d \rho = D A_d\sqrt{3}\,T^{-1/(2d)}.
\end{equation}
\end{proposition}

\paragraph{Constraint halfspaces.}

For each \(m=1,\dots,M\) and \(i=1,\dots,n\), define the projection of \(z_i^{(m)}\) onto the next sphere \(C_{m+1}\) as $p_i^{(m)} = \PP_{\BB(0, D r_{m+1})} (z_i^{(m)})$. Note that
\begin{equation}
p_i^{(m)} = Dr_{m+1} u_i^{(m)}.
\end{equation}
Let \(H_i^{(m)}\) be the closed halfspace bounded by the tangent hyperplane to $C_{m+1}$ at \(p_i^{(m)}\) and containing the origin:
\begin{equation}
H_i^{(m)} = \left\{x\in\mathbb R^d: x^T u_i^{(m)} \le Dr_{m+1} \right\}.
\end{equation}

The following proposition can be concluded.
\begin{proposition}
\label{prop:z_i_belongs_1}
\begin{equation}
\mathbb B(0,Dr_{m+1}) \subset H_i^{(m)} \quad \forall 1 \le i \le n, 1 \le m \le M.
\end{equation}    
\end{proposition}

The halfspace constraint \(H_i^{(m)}\) removes from \(C_m\) precisely the spherical cap centered at \(z_i^{(m)}\) with half-angle \(\vartheta_m\). Because the vectors in each layer are \(\rho\)-separated and \(\vartheta_m\le \rho\), the cap removed by \(H_i^{(m)}\) contains no other vector in $\chi_m$. Thus, we get the following proposition.
\begin{proposition}
\label{prop:z_i_belongs_2}
\begin{equation}
z_j^{(m)}=Dr_m u_j^{(m)}\in H_i^{(m)} \text{ for } 1 \le i, j \le n, i \neq j, 1 \le m \le M.
\end{equation}    
\end{proposition}

Moreover, since \(u_i^{(m)}\) is a unit normal vector to \(H_i^{(m)}\),
\begin{equation}
\label{eqn:projection_cost}
\dist\left(z_i^{(m)},H_i^{(m)}\right) = Dr_m-Dr_{m+1} = \frac{D}{2M} = \frac{D}{2}T^{-1/d}.
\end{equation}

\begin{figure}[t]
\centering
\begin{tikzpicture}[scale=1.0,>=Latex]
    % Radii: M=2 => r1=1, r2=3/4, r3=1/2
    \def\D{4.0}
    \def\rone{4.0}   % D
    \def\rtwo{3.0}   % 3D/4
    \def\rthree{2.0} % D/2
    \pgfmathsetmacro{\thcap}{acos(\rtwo/\rone)} % arccos(3/4)

    % Concentric circles
    \draw[thick] (0,0) circle (\rone);
    \draw[thick] (0,0) circle (\rtwo);
    \draw[thick] (0,0) circle (\rthree);

    \fill (0,0) circle (1.2pt);
    \node[below left] at (0,0) {$0$};

    % Labels for circles
    \node[font=\small] at (135:\rone+0.35) {$C_1$};
    \node[font=\small] at (145:\rtwo+0.35) {$C_2$};
    \node[font=\small] at (155:\rthree+0.35) {$C_3$};

    % Outer layer: z_i^(1), 6 equally spaced
    \foreach \i in {1,...,6}{
        \pgfmathsetmacro{\angA}{(\i-1)*60}
        \coordinate (zA\i) at (\angA:\rone);
        \draw[gray!40,dashed] (0,0) -- (zA\i);
        \fill[blue] (zA\i) circle (2.3pt);
        \node[font=\scriptsize, text=blue!70!black]
            at (\angA:\rone+0.42)
            {$z_{\i}^{(1)}$};
    }

    % Middle layer: z_i^(2), rotated so z_1^(2) aligns with z_6^(1)
    \foreach \i in {1,...,6}{
        \pgfmathsetmacro{\angB}{-60+(\i-1)*60}
        \coordinate (zB\i) at (\angB:\rtwo);
        \fill[red] (zB\i) circle (2.3pt);
        \node[font=\scriptsize, text=red!70!black]
            at (\angB:\rtwo-0.42)
            {$z_{\i}^{(2)}$};
    }

    % Radial transition from z_6^(1) to z_1^(2)
    \draw[purple, very thick, <->] (300:\rtwo) -- (300:\rone);
    \node[font=\scriptsize, text=purple] at (300:3.55)
        {$\frac{D}{2M}=\frac{D}{4}$};

    % Projection point q_1^(1) from z_1^(1) onto H_1^(1)
    \coordinate (q) at (0:\rtwo);

    % Projection arrow
    \draw[orange!90!black, very thick, ->] (zA1) -- (q);

    % Projection point
    \fill[orange!90!black] (q) circle (2.6pt);

    % Projection label on the top-left
    \node[font=\scriptsize, text=orange!90!black, anchor=south east]
        at ($(q)+(-0.25,0.45)$)
        {$\PP_{H_1^{(1)}}(z_1^{(1)})$};

    % Tangent line: boundary of H_1^(1), vertical at x = rtwo
    \draw[orange!80!black, thick]
        (\rtwo,-3.1) -- (\rtwo,3.1);

    \node[font=\scriptsize, text=orange!80!black, anchor=west]
        at ($(q)+(0.15,2.85)$)
        {$\partial H_1^{(1)}$};

    % Removed cap on C_1 around z_1^(1)
    \draw[orange!90!black, very thick]
        ({-\thcap}:\rone)
        arc[start angle={-\thcap}, end angle={\thcap}, radius=\rone];
\end{tikzpicture}
\caption{Nested construction in $d=2$ with $M=2$, showing the cap removal and projection associated with \(z_1^{(1)}\). Note the rotation of the vectors in the second layer to ensure that $z_6^{(1)}$ and $z_1^{(2)}$ are directly below each other.}
\label{fig:nested-construction-2d}
\end{figure}
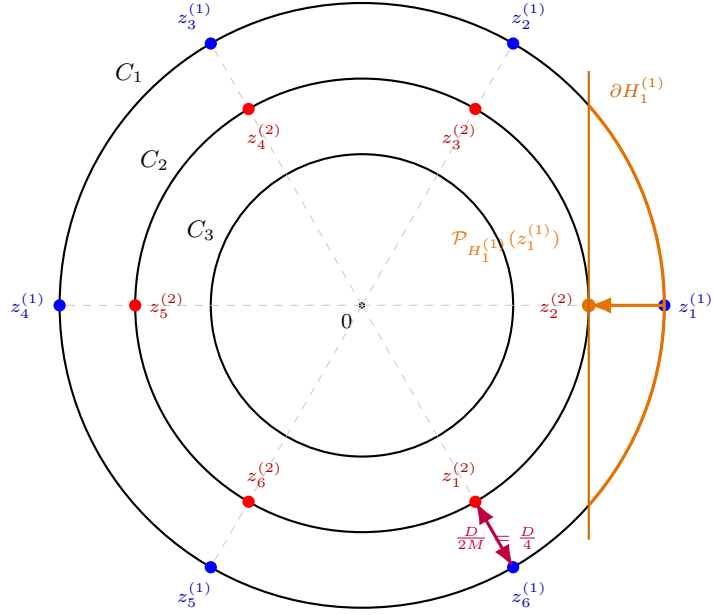

The entire construction of the constraint half-planes is depicted in $d = 2$ in~\ref{fig:nested-construction-2d}.

\paragraph{Constraint sets.} We next define the constraint sets. Recall that the constraint sets are given by $\mathcal{G}_t = \{x \in \mathcal{X}: g_t(x) \le 0\}$ and $S_t = \mathcal{G}_t \cap S_{t - 1}$ for $1 \le t \le T$, where we take $S_0 = \mathcal{X}$. Note that the $S_t$'s are nested. We first define the $S_t$'s and then define the corresponding constraint functions ($g_t$'s) from these sets.

From the halfspaces, we first define sets $K_p$, for $1 \le p \le P = |Z| = Mn$. The $S_t$'s consist of $P$ phases and in phase $p$, we show the same constraint set  $K_p$ repeatedly for $\Delta = \lfloor \frac TP \rfloor$ many time-slots. This accounts for the constraints shown in the $T$ time slots.

Recall that $P=Mn$. We index phases by
\begin{equation}
p=(m-1)n+i, \qquad m=1,\dots,M, \qquad i=1,\dots,n.
\end{equation}
Write
\begin{equation}
H_p=H_i^{(m)}
\qquad
\text{and}
\qquad
z_p=z_i^{(m)}.
\end{equation}
We define the distinct constraint sets recursively by
\begin{equation}
K_0=\mathbb B(0,Dr_1); \qquad K_p=K_{p-1}\cap H_p \quad \text{ for } p=1,\dots,P.
\end{equation}
Since \(K_0\) is convex and every \(H_p\) is a halfspace, each \(K_p\) is convex and hence a valid constraint set.

The construction has the following useful feasibility property. Consider $z_p \in Z, 1 \le p \le P$, where $p = (m - 1) n + i$ for some $m, i$. Note that the halfspace constraints in the $l$th layer does not remove any point in the subsequent layer (since it is tangential to the subsequent layer). In other words, using Proposition~\ref{prop:z_i_belongs_1} we have $\BB(0, D r_m) \subseteq \BB(0, D r_{l + 1}) \subset H_j^{(l)} \forall 1 \le j \le n, 1 \le l < m \le M$. From this, we have
\begin{equation}
    z_i^{(m)} \in H_{j}^{(l)} \forall 1 \le i, j \le n, 1 \le l < m \le M.
\end{equation}
Also recall Proposition~\ref{prop:z_i_belongs_2}, i.e. $z_i^{(m)} \in H_j^{(m)} \forall i \neq j, 1 \le i, j \le n, 1 \le m \le M$.

Since $K_p = K_{p - 1} \cap H_p$, and $z_i^{(m)} \in H_j^{(l)}$ if either $l < m$ or $j < i$, the following proposition of feasibility follows.
\begin{proposition}
\begin{equation}
    z_p \in K_{p - 1}.
\end{equation}
\end{proposition}

Once the set $K_p$ is shown in phase $p = (m - 1) n + i$ however, $z_p \not \in K_{p}$ and a constraint violation incurred, as below.
\begin{equation}
\dist\left(z_i^{(m)},K_{p}\right) \ge \dist\left(z_i^{(m)},H_i^{(m)}\right) = \frac{D}{2M} = \frac{D}{2} T^{-1/d}.
\label{eqn:one_step_ccv}
\end{equation}

Thus phase $p$ has a vector $z_p$ that is feasible immediately before the constraint is imposed (i.e. $z_p \in K_{p - 1}$) and violates the new constraint by order \(T^{-1/d}\) immediately after the constraint is imposed (i.e. since $z_p \not \in K_p$).

The number of distinct phases is $
P = Mn$. Recall that
\begin{equation}
\Delta=\left\lfloor \frac{T}{P}\right\rfloor = \Theta \left(T^{\frac{d-1}{2d}}\right),
\end{equation}
such that within the $p$th phase, the constraint set $K_p$ is repeated for $\Delta$ time-steps. More formally, for $t=(p-1)\Delta+1,\dots,p\Delta$, we take
\begin{equation}
S_t=K_p.
\end{equation}
If \(P\Delta<T\), the remaining time-steps are filled by repeating \(K_P\).

The sequence begins as
\begin{align}
S_1=\cdots=S_\Delta
&=
K_1
=
K_0\cap H_1^{(1)},\\
S_{\Delta+1}=\cdots=S_{2\Delta}
&=
K_2
=
K_1\cap H_2^{(1)},\\
&\vdots \nonumber
\end{align}
and continues until all \(Mn\) halfspaces have been introduced.

\paragraph{Cost functions.} We now define the cost functions for the COCO instance. The cost functions are such that $\AA$ will play actions which move from $z_p$ to $z_{p+1}$ within phase $p$. Essentially, the cost functions are oriented such that the actions played by $\AA$ follows the trajectory along $z_1, z_2, \dots, z_P$.

\begin{figure}[t]
\centering
\begin{tikzpicture}[scale=1.0,>=Latex]

    % Basic parameters
    \def\Rone{4.0}
    \def\Rtwo{3.0}
    \def\angone{0}
    \def\angtwo{60}

    % Circles
    \draw[thick] (0,0) circle (\Rone);
    \draw[thick] (0,0) circle (\Rtwo);

    \fill (0,0) circle (1.2pt);
    \node[below left] at (0,0) {$0$};

    \node[font=\small] at (135:\Rone+0.35) {$C_1$};
    \node[font=\small] at (145:\Rtwo+0.35) {$C_2$};

    % Main points
    \coordinate (z1) at (\angone:\Rone);
    \coordinate (q1) at (\angone:\Rtwo);
    \coordinate (z2) at (\angtwo:\Rone);

    % Intermediate iterates
    \coordinate (x3) at ($(q1)!0.22!(z2)$);
    \coordinate (x4) at ($(q1)!0.48!(z2)$);
    \coordinate (x5) at ($(q1)!0.74!(z2)$);

    % Reference rays
    \draw[gray!50,dashed] (0,0) -- (z1);
    \draw[gray!50,dashed] (0,0) -- (z2);

    % Tangent line at q_1
    \draw[orange!85!black, thick] (\Rtwo,-2.8) -- (\Rtwo,2.8);
    \node[font=\scriptsize, text=orange!85!black]
        at (\Rtwo+0.65,2.45)
        {$\partial H_1$};

    % Projection step z1 -> q1
    \draw[orange!85!black, very thick, ->] (z1) -- (q1);
    \node[font=\scriptsize, text=orange!85!black, anchor=east]
        at ($(q1)+(-0.18,-0.28)$)
        {$x_2=\PP_{K_1}(z_1)$};

    % Iteration arrows
    \draw[blue!70!black, thick, ->] (q1) -- (x3);
    \draw[blue!70!black, thick, ->] (x3) -- (x4);
    \draw[blue!70!black, thick, ->] (x4) -- (x5);
    \draw[blue!70!black, thick, ->] (x5) -- (z2);

    % Points
    \fill[blue] (z1) circle (2.5pt);
    \fill[orange!90!black] (q1) circle (2.5pt);
    \fill[blue] (z2) circle (2.5pt);
    \fill[red] (x3) circle (2.1pt);
    \fill[red] (x4) circle (2.1pt);
    \fill[red] (x5) circle (2.1pt);

    % Labels
    \node[font=\small, text=blue!70!black] at (4.7,0.0)
        {$x_1=z_1$};

    \node[font=\small, text=red!80!black, anchor=east]
        at ($(x3)+(0.25,-0.4)$)
        {$x_3$};

    \node[font=\small, text=red!80!black, anchor=west]
        at ($(x4)+(0.05,0.01)$)
        {$x_4$};

    \node[font=\small, text=red!80!black, anchor=west]
        at ($(x5)+(0.05,0.01)$)
        {$x_5$};

    \node[font=\small, text=blue!70!black]
        at ($(z2)+(0.55,0.25)$)
        {$x_6=z_2$};

    \node[font=\scriptsize]
        at ($(z2)+(-0.28,0.35)$)
        {$z_2$};
\end{tikzpicture}
\caption{Actions played during the first phase from \(z_1\) to \(z_2\).}
\label{fig:first-phase-motion}
\end{figure}
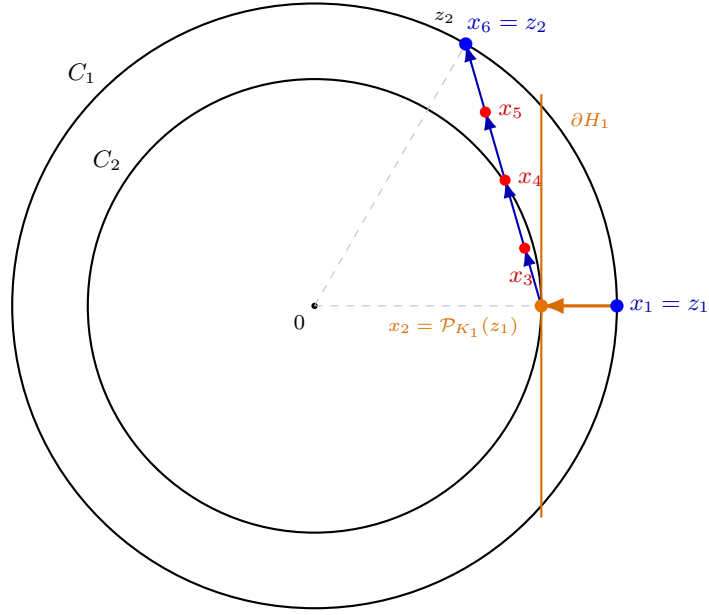

We will assume that the initial action satisfies $x_1 = z_1$.

For a phase $p=(m-1)n+i$, define the time-instant $\tau_p = (p-1)\Delta+1$ as the phase head. Thus, at time-instant $ \tau_p$, the constraint set $S_{\tau_p} = K_{p}$ is shown and a constraint violation will be incurred by $\AA$ as in~\eqref{eqn:one_step_ccv}.

Recall that
\begin{equation}
z_p=z_i^{(m)}=Dr_m u_i^{(m)}.    
\end{equation}
Let
\begin{equation}
q_p = p_i^{(m)} = D r_{m+1}u_i^{(m)},
\label{eqn:q_p_definition}
\end{equation}
which is the projection of \(z_p\) onto the tangent halfspace \(H_i^{(m)}\).

Using~\eqref{eqn:one_step_ccv}, note that the projection distance (i.e. the distance between $z_p$ and $q_p$) satisfies
\begin{equation}
    \|z_p - \PP_{K_p} (z_p) \| = \|z_p - q_p\| \ge \frac D2 T^{-\frac 1d}.
\label{eqn:per_step_projection_distance}
\end{equation}

For each phase $p$, set the cost function during the phase-head, $f_{\tau_p}(x)=0$. It is left to define the remaining cost functions within each of the phases.

Now, assume inductively that $x_{\tau_p}=z_p$. The first update in phase \(p\) is therefore $y_{\tau_p}=x_{\tau_p}=z_p$ and hence the proposition below follows. 
\begin{proposition}
    Assuming $x_{\tau_p} = z_p$, $x_{\tau_p+1} = \PP_{K_p}(z_p) = q_p$, as defined in~\eqref{eqn:q_p_definition}.
\end{proposition}

Now suppose \(p<P\). We want the cost functions to be such that $\AA$ moves from \(q_p\) to \(z_{p+1}\) during the remaining \(\Delta-1\) slots of phase \(p\). Define
\begin{equation}
\ell_p=\|z_{p+1}-q_p\|.    
\end{equation}
If $\ell_p = 0$, set all the cost functions in that phase to be $0$. Otherwise, define
\begin{equation}
a_p=\frac{z_{p+1}-q_p}{\|z_{p+1}-q_p\|}.
\end{equation}
Within phase $p$, the unit vector \(a_p\) points from the projection vector \(q_p\) toward the next vector \(z_{p+1}\).

Since \(q_p, z_{p+1} \in K_p\) and \(K_p\) is convex, the following proposition holds for the line segment from $q_p$ to $z_{p+1}$, denoted by $[q_p, z_{p+1}]$.
\begin{proposition}
\begin{equation}
[q_p,z_{p+1}] \subseteq K_p.
\end{equation}  
\label{prop:line_segment_belongs}
\end{proposition}
Therefore, as long as the gradient directions are along this segment, the projection onto \(K_p\) does not affect the trajectory.

Recall that $G$ is the Lipschitz constant of the cost and constraint functions. We now choose nonnegative scalars $\gamma_t\in[0,G]$ for $t=\tau_p+1,\dots,\tau_p + \Delta - 1$ such that
\begin{equation}
\sum_{t=\tau_p+1}^{p\Delta}\eta_t\gamma_t=\ell_p.    
\label{eqn:gamma_requirement}
\end{equation}

We will argue that scalars $\gamma_t$ exist satisfying~\eqref{eqn:gamma_requirement}. From Proposition~\ref{prop:small_z_distance}, we have
\begin{equation}
    \ell_p \overset{(a)}{\le} \left\|z_{p+1}-z_p\right\| + \left\|q_p -z_p\right\| \overset{(b)}{\le} D A_d\rho + \frac{D}{2M} \overset{(c)}{\le} D B_d T^{-1/(2d)},
    \label{eqn:equality_upper_bound}
\end{equation}
where $(a)$ follows from the triangle inequality, $(b)$ uses Proposition~\ref{prop:small_z_distance} and~\eqref{eqn:per_step_projection_distance}, and $(c)$ uses $\frac{1}{2M} = \frac{\rho^2}{6}$ from~\eqref{eqn:rho_definition} and $A_d \rho + \frac{\rho^2}{6} \le B_d \rho$ for some $B_d \ge 1$. Next, note that since $\gamma_t$'s are real numbers between $[0, G]$, the maximum value which $\sum_{t=\tau_p+1}^{p\Delta}\eta_t \gamma_t$ can take is:
\begin{equation}
o_T = \sum_{t=\tau_p+1}^{p\Delta}\eta_tG = \sum_{t=\tau_p+1}^{p\Delta}\frac{2 D}{\sqrt t} \ge \frac{2 (\Delta-1)D}{\sqrt T} = \Omega\left(T^{-1/(2d)}\right). 
\label{eqn:equality_lower_bound}
\end{equation}
Thus, we have shown that $\sum_{t=\tau_p+1}^{p\Delta}\eta_t \gamma_t$ can take values in $[0, o_T]$, where $o_T = \Omega\left(T^{-1/(2d)}\right)$.

Equation~\eqref{eqn:equality_upper_bound} shows that $\ell_p$ is also $O(T^{-1/(2d)})$ while~\eqref{eqn:equality_lower_bound} shows that the scalars $\gamma_t$ can be chosen to have any value which is $O(T^{-1/(2d)})$. Thus, the value of $\tilde{c}_d$ from the definition of $n$ in~\eqref{eqn:n_definition} can be chosen to be sufficiently small such that these scalars $\gamma_t$ exist.

We already defined the cost functions during the phase-heads, i.e. the $\tau_p$'s. For the remaining time steps in phase \(p\), define the linear losses
\begin{equation}
f_t(x)=-\gamma_t a_p^\top x,
\qquad
t=\tau_p+1,\dots,p\Delta.    
\end{equation}
Each \(f_t\) is convex and \(G\)-Lipschitz because
\begin{equation}
\|\nabla f_t\|=\gamma_t\le G
\end{equation}
and $|a_t| = 1$.

The corresponding actions played by algorithm $\AA$ are
\begin{equation}
y_t = x_t-\eta_t\nabla f_t(x_t) = x_t+\eta_t\gamma_t a_p.    
\end{equation}

Within phase $p$, starting from \(x_{\tau_p+1}=q_p\), the actions played by $\AA$ are along the line segment $[q_p, z_{p+1}]$ from \(q_p\) to \(z_{p+1}\). Using Proposition~\ref{prop:line_segment_belongs}, this entire segment is contained in \(K_p\), and thus the projection step keeps $y_t$ the same as $x_{t + 1}$, as below:
\begin{equation}
x_{t+1}=\PP_{K_p}(y_t)=y_t.
\end{equation}
By the choice of the coefficients \(\gamma_t\),
\begin{equation}
x_{p\Delta+1}
=
q_p+
\left(
\sum_{t=\tau_p+1}^{p\Delta}\eta_t\gamma_t
\right)a_p
=
q_p+\ell_p a_p
=
z_{p+1}.
\end{equation}
The movement from $z_1$ to $z_2$ in Phase 1 for $d = 2$ is depicted schematically in Fig.~\ref{fig:first-phase-motion}.

For the final phase $p = P$ and time-slots after that, there is no `next' vector; we may set all remaining losses to zero. Thus, set $f_t \equiv 0$ for $t = \tau_p + 1, \dots, T$.

Therefore the following invariant holds between phases:
\begin{equation}
x_{\tau_p}=z_p
\quad\Longrightarrow\quad
x_{p\Delta+1}=z_{p+1}
\qquad
\text{for every }p<P.
\end{equation}

This completes the construction. Finally, we are ready to complete the proof of Theorem \ref{thm:main}.

\subsection{Completing the Proof of Theorem \ref{thm:main}}
\begin{proof} First, note from~\eqref{eqn:per_step_projection_distance} that the projection distance incurred at the phase-heads $\tau_p$ for $p = 1, \dots, P$ is
\begin{equation}
    \|x_{\tau_p} - \PP_{S_{\tau_p}} (x_{\tau_p})\| = \|z_p - q_p \| = \frac{D}{2} T^{-1/d}
    \label{eqn:proj_cost_bound}
\end{equation}
and that there are a total of $P = Mn = \tilde{c}_d T^{1/d + (d-1)/(2d)} = \tilde{c}_d T^{(d+1)/(2d)}$ such phases.

Recall that we have defined only the constraint sets (i.e. $S_t$'s) only and are yet to define the constraint functions themselves. We take the constraint functions to be
\begin{equation}
    g_t(x) = G \cdot \|x - \PP_{S_t}(x)\|,
    \label{eqn:g_t_definition}
\end{equation} which is essentially $G$ times the projection distance from the set $S_t$. These constraint functions (i.e. $g_t$'s) are convex and $G$--Lipschitz.

From this, 
\begin{equation}
    \sum_{t = 1}^T [g_t(x_t)]_+ \overset{(a)}{\ge} \sum_{p = 1}^P g_{\tau_p}(x_{\tau_p}) \overset{(b)}{=} \sum_{p = 1}^P G \|x_{\tau_p} - \PP_{S_{\tau_p}} (x_{\tau_p})\| \overset{(c)}{\ge} \tilde{c}_d \frac{G D}{2} T^{\frac{d-1}{2d}},
\end{equation}
where $(a)$ follows from removing the constraint violations of the terms which are not phase-heads (i.e. we retain the constraint violation terms only when $t = \tau_p$ for some $1 \le p \le P$), $(b)$ uses the definition of $g_t$'s from~\eqref{eqn:g_t_definition}, and $(c)$ uses~\eqref{eqn:proj_cost_bound}. \end{proof}

\section{Proof of Theorem~\ref{thm:unit_vector_theorem}}

\label{sec:unit_vector_proof}

For the inductive construction it is useful to demand two additional properties. First, we maintain the ordered collection to be cyclic, such that $\theta_{\mathcal S}\left(u_{N_d(\rho)},u_1\right)\le A_d\rho$. This additional cyclic local property is what allows the construction to return to the same place where it started before passing to the next higher dimension. Second, we will also maintain the property that $N_d(\rho)$ is even.

We prove, by induction on the number of dimensions \(k\), the following strengthened statement:
for every \(k\ge 2\), there exist constants \(A_k\ge 1\) and \(c_k>0\)
such that for every \(\rho\in(0,\pi/2]\), the constructed collection
\(\mathcal U_k(\rho)\) satisfies the three properties in the theorem,
the cyclic local transition bound
\begin{equation}
\theta_{\mathcal S}(u_{N_k(\rho)},u_1)\le A_k\rho,
\end{equation}
and the parity property that \(N_k(\rho)\) is even.

\subsection{Construction of the unit vectors}

First, fix any $\rho \in (0, \frac \pi 2]$.

\noindent \textbf{Base case \(k=2\).}

Let $N_2(\rho)=2\left\lfloor \frac{\pi}{\rho}\right\rfloor$ and $\gamma(\rho)=\frac{2\pi}{N_2(\rho)}$. Take the orthonormal basis \((e_1,e_2)\) of \(\mathbb R^2\), where \(e_1=(1,0)\) and \(e_2=(0,1)\). Define
\begin{equation}
\mathcal U_2(\rho)=\left(u_1,\dots,u_{N_2(\rho)}\right),
\end{equation}
where
\begin{equation}
u_j=\cos((j-1)\gamma(\rho))e_1+\sin((j-1)\gamma(\rho))e_2,
\qquad
j=1,\dots,N_2(\rho).
\end{equation}

Since \(\rho\le \pi/2\), we have \(\pi/\rho\ge 2\). Therefore
\begin{equation}
\label{eqn:gamma-bound-new}
\rho\le \gamma(\rho)\le 2\rho.
\end{equation}

Note that $N_2(\rho)$ is even, by construction. The construction is depicted in Fig.~\ref{fig:unit-vectors-s1}.

\begin{figure}[t]
\centering
\begin{tikzpicture}[scale=1.1,>=Latex]
    \def\R{3.0}
    \def\N{6}

    % Circle
    \draw[thick] (0,0) circle (\R);
    \fill (0,0) circle (1.2pt);
    \node[below left] at (0,0) {$0$};

    % Points u_i
    \foreach \i in {1,...,6}{
        \pgfmathsetmacro{\ang}{(\i-1)*360/\N}
        \coordinate (u\i) at (\ang:\R);
        \draw[gray!60,dashed] (0,0) -- (u\i);
        \fill[blue] (u\i) circle (2.2pt);
        \node[font=\small] at (\ang:\R+0.45) {$u_{\i}$};
    }

    % Adjacent angular separation
    \draw[red!80!black, thick, <->] (0:1.15)
        arc[start angle=0, end angle=60, radius=1.15];
    \node[red!80!black, font=\small, anchor=west]
        at ($(30:0.8)+(0.35,0)$)
        {$\gamma = \frac{2\pi}{6}\ge \rho$};
\end{tikzpicture}
\caption{Base construction on $\mathcal S^1$ with six equally spaced directions for $d = 2$.}
\label{fig:unit-vectors-s1}
\end{figure}
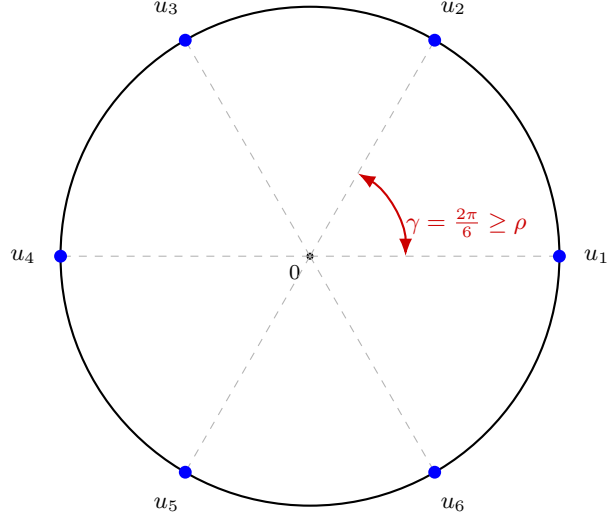

\noindent \textbf{Recursive step.} Fix \(k\ge 3\). Assume recursively that the construction has already been carried out in dimension \(k-1\), and that for every \(\rho \in(0,\pi/2]\), the ordered collection
\begin{equation}
\mathcal U_{k-1}(\rho) = \left(v_1,\dots,v_{N_{k-1}(\rho)}\right) \subset \mathcal S^{k-1}
\end{equation}
satisfies the properties in the statement of the theorem.

Let \(e_1,\dots,e_k\) denote the standard basis unit vectors of \(\mathbb R^k\), i.e. the vector $e_i \in \RR^k$ is a unit vector which lies along the $i$th coordinate axis. For \(z=(z_1,\dots,z_{k-1})\in\mathcal S^{k-1}\subset\mathbb R^{k-1}\), define the canonical embedding into \(\mathbb R^k\) by
\begin{equation}
\Lambda_k(z)=(z_1,\dots,z_{k-1},0)\in\mathbb R^k.
\end{equation}

Next, we note the following propositions which follow from the definition of $\Lambda_k$.
\begin{proposition}
\begin{equation}
\Lambda_k(z)\in\mathcal S^{k} \forall z \in \mathcal S^{k-1}.
\end{equation}    
\end{proposition}

\begin{proposition}
\begin{equation}
\Lambda_k(z)^T e_k=0 \forall z \in \mathcal S^{k-1}.
\end{equation}    
\end{proposition}

We next split the construction into two regimes.

\noindent\textbf{Small-scale regime \(0<\rho\le \pi/4\).}

\noindent \textit{Intuition for the induction from $k = 2$ to $k = 3$.} Assume that the two-dimensional construction is available. In three-dimensions, we consider $N_2(\eta)$ many longitudes, where $\eta = 2 \rho$. These longitudes are determined by the equatorial vector, which is obtained from the lower-dimensional construction. Once the longitudes are defined, we populate each of these longitudes with vectors with polar angles $\frac{\pi}{4} \le \alpha_l \le \frac{3 \pi}{4}$ and separation $\rho$. In other words, we populate each longitude with $L(\rho)$ vectors (defined later) from the latitude corresponding to a polar angle $\frac{\pi}{4}$ to the latitude corresponding to a polar angle $\frac{3 \pi}{4}$. To obtain the ordering, we traverse the odd longitudes top-to-bottom and the even longitudes bottom-to-top, to ensure that consecutive vectors are not too far apart. This is depicted schematically in Fig.~\ref{fig:induction-k2-to-k3}.

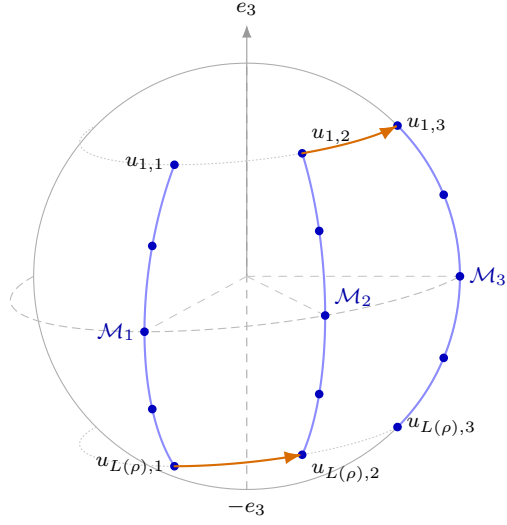
\begin{figure}[t]
\centering
\begin{tikzpicture}[
    scale=1.2,
    x={(1cm,0cm)},
    y={(0.48cm,0.26cm)},
    z={(0cm,1cm)},
    >=Latex,
    line cap=round,
    line join=round
]

% Illustrative parameters only
\pgfmathsetmacro{\R}{2.35}
\pgfmathsetmacro{\M}{8}
\pgfmathsetmacro{\L}{5}

% Three visible/front longitudes, starting from the middle and going rightwards
\def\shownlongitudes{7,8,1}

% Polar band alpha in [pi/4, 3pi/4]
\pgfmathsetmacro{\alphaTop}{45}
\pgfmathsetmacro{\alphaBot}{135}
\pgfmathsetmacro{\rTop}{\R*sin(\alphaTop)}
\pgfmathsetmacro{\zTop}{\R*cos(\alphaTop)}
\pgfmathsetmacro{\rBot}{\R*sin(\alphaBot)}
\pgfmathsetmacro{\zBot}{\R*cos(\alphaBot)}

\coordinate (O) at (0,0,0);
\coordinate (N) at (0,0,\R);
\coordinate (S) at (0,0,-\R);

% Axis of the sphere
\draw[gray!70,dashed] (S) -- (N);
\draw[-Latex,gray!80] (O) -- (0,0,1.18*\R);

% Sphere silhouette
\begin{scope}[canvas is xz plane at y=0]
  \draw[gray!65] (0,0) circle[radius=\R];
\end{scope}

% Visible front half of equator
\draw[gray!55,densely dashed]
  plot[domain=180:360,samples=80,variable=\phi]
  ({\R*cos(\phi)},{\R*sin(\phi)},0);

% Visible front halves of boundary latitudes
\draw[gray!45,densely dotted]
  plot[domain=180:360,samples=80,variable=\phi]
  ({\rTop*cos(\phi)},{\rTop*sin(\phi)},\zTop);

\draw[gray!45,densely dotted]
  plot[domain=180:360,samples=80,variable=\phi]
  ({\rBot*cos(\phi)},{\rBot*sin(\phi)},\zBot);

% Shown longitudes and points
\foreach \j in \shownlongitudes {
  \pgfmathsetmacro{\phi}{(\j-1)*360/\M}

  % Curved longitude segment in the populated band
  \draw[blue!45,thick]
    plot[domain=45:135,samples=50,variable=\a]
    ({\R*sin(\a)*cos(\phi)},
     {\R*sin(\a)*sin(\phi)},
     {\R*cos(\a)});

  % Embedded lower-dimensional equatorial vector
  \coordinate (v-\j) at ({\R*cos(\phi)},{\R*sin(\phi)},0);
  \draw[gray!45,dashed] (O) -- (v-\j);
  \fill[black] (v-\j) circle[radius=1.1pt];

  % Points u_{\ell,j}
  \foreach \ell in {1,...,5} {
    \pgfmathsetmacro{\alpha}{45+(\ell-1)*90/(\L-1)}
    \coordinate (u-\ell-\j)
      at ({\R*sin(\alpha)*cos(\phi)},
          {\R*sin(\alpha)*sin(\phi)},
          {\R*cos(\alpha)});
    \fill[blue!75!black] (u-\ell-\j) circle[radius=1.35pt];
  }
}

% Curved transition arrows between visible longitudes
\draw[-Latex,orange!85!black,thick]
  plot[domain=270:315,samples=40,variable=\phi]
  ({\rBot*cos(\phi)},{\rBot*sin(\phi)},\zBot);

\draw[-Latex,orange!85!black,thick]
  plot[domain=315:360,samples=40,variable=\phi]
  ({\rTop*cos(\phi)},{\rTop*sin(\phi)},\zTop);

% Labels for axis
\node[anchor=south] at (0,0,1.18*\R) {$e_3$};
\node[anchor=north] at (S) {$-e_3$};

% Labels for the three longitudes
\node[blue!65!black,anchor=east] at (u-3-7) {$\mathcal M_1$};
\node[blue!65!black,anchor=south west] at (u-3-8) {$\mathcal M_2$};
\node[blue!65!black,anchor=west] at (u-3-1) {$\mathcal M_3$};

% Top and bottom point labels on each longitude
\node[anchor=east] at (u-1-7) {$u_{1,1}$};
\node[anchor=east] at (u-5-7) {$u_{L(\rho),1}$};

\node[anchor=south west] at (u-1-8) {$u_{1,2}$};
\node[anchor=north west] at (u-5-8) {$u_{L(\rho),2}$};

\node[anchor=west] at (u-1-1) {$u_{1,3}$};
\node[anchor=west] at (u-5-1) {$u_{L(\rho),3}$};

\end{tikzpicture}
\caption{Induction from \(k=2\) to \(k=3\). The embedded equatorial vectors
\(\Lambda_3(v_1),\Lambda_3(v_2),\Lambda_3(v_3)\) determine longitudes
\(\mathcal M_1,\mathcal M_2,\mathcal M_3\), which are populated by points
\(u_{\ell,1},u_{\ell,2},u_{\ell,3}\) with polar angles
\(\pi/4\le \alpha_\ell\le 3\pi/4\).}
\label{fig:induction-k2-to-k3}
\end{figure}

\noindent \textit{Formal construction for general $k$.} For $k \ge 3$, set $\eta=2\rho$. Since \(\rho\le \pi/4\), we have \(\eta\in(0,\pi/2]\), so the lower-dimensional construction is available. Write
\begin{equation}
\mathcal U_{k-1}(\eta)
= \mathcal U_{k-1}(2 \rho) =
\left(v_1,\dots,v_M\right),
\qquad
M=N_{k-1}(\eta).
\label{eqn:lower_dimensional_vectors_defn}
\end{equation}

Each \(v_j\in\mathcal S^{k-1}\) determines a longitude, or meridian, in \(\mathcal S^{k}\):
\begin{equation}
\mathcal M_j
=
\left\{
\cos(\alpha)e_k+\sin(\alpha)\Lambda_k(v_j):
\alpha\in[0,\pi]
\right\}.
\end{equation}
Thus, the lower-dimensional vectors \(v_j\) should be interpreted as \textit{longitude} vectors. On each of the longitudes, we construct $L(\rho)$ many vectors as below.

Define the polar levels
\begin{equation}
L(\rho)=\left\lfloor \frac{\pi}{2\rho}\right\rfloor+1
\end{equation}
and
\begin{equation}
\alpha_\ell=\frac{\pi}{4}+(\ell-1)\rho,
\qquad
\ell=1,\dots,L(\rho).
\end{equation}

The proposition below follows since $\alpha_l = \frac{\pi}{4} + (l-1)\rho \le \frac{\pi}{4} + (L(\rho)-1)\rho \le \frac {\pi}{4} + \frac {\pi}{2}$.
\begin{proposition}
\label{prop:alpha_bounds}
\begin{equation}
\frac{\pi}{4}\le \alpha_\ell\le \frac{3\pi}{4},
\end{equation}    
\end{proposition}

For each longitude \(j=1,\dots,M\) and each polar level \(\ell=1,\dots,L(\rho)\), define
\begin{equation}
\label{eqn:u-ell-j-new}
u_{\ell,j}
=
\cos(\alpha_\ell)e_k+\sin(\alpha_\ell)\Lambda_k(v_j).
\end{equation}

The order is longitude-first. We traverse the first longitude from
\(\alpha_1\) to \(\alpha_{L(\rho)}\), the second longitude from
\(\alpha_{L(\rho)}\) back to \(\alpha_1\), and so on, alternating vector
from one longitude to the next. This gives us an ordering:
\begin{equation}
\begin{aligned}
\mathcal U_k(\rho)
=
\big(&
u_{1,1},u_{2,1},\dots,u_{L(\rho),1},\\
&
u_{L(\rho),2},u_{L(\rho)-1,2},\dots,u_{1,2},\\
&
u_{1,3},u_{2,3},\dots,u_{L(\rho),3},\\
&
u_{L(\rho),4},u_{L(\rho)-1,4},\dots,u_{1,4},\\
&
\dots,\\
&
u_{1,M-1},u_{2,M-1},\dots,u_{L(\rho),M-1},\\
&
u_{L(\rho),M},u_{L(\rho)-1,M},\dots,u_{1,M}
\big).
\end{aligned}
\end{equation}
Equivalently, the construction sweeps through the longitudes
\(j=1,\dots,M\), traversing consecutive longitudes in opposite polar
vectors.

In this regime,
\begin{equation}
\label{eqn:N-small-regime}
N_k(\rho) = L(\rho)M =
L(\rho)N_{k-1}(2\rho).
\end{equation}
It follows that $N_k(\rho)$ is even since$N_{k - 1} (2 \rho)$ is even.

\noindent\textbf{The large-scale regime \(\pi/4<\rho\le \pi/2\).}

In this regime, we take exactly two vectors ($N_k(\rho)=2$),
\begin{equation}
u_1= \cos\left(\frac{\pi}{4}\right)e_k + \sin\left(\frac{\pi}{4}\right)e_1,
\end{equation}
\begin{equation}
u_2= \cos\left(\frac{3\pi}{4}\right)e_k + \sin\left(\frac{3\pi}{4}\right)e_1.
\end{equation}
Then define
\begin{equation}
\mathcal U_k(\rho)=(u_1,u_2).
\end{equation}
The construction in the large-scale regime does not use the lower dimensional vectors. It follows that $N_k(\rho)$ is even.

This completes the recursive construction. From here on, we take that for every \(k\ge 2\) and every \(\rho\in(0,\pi/2]\), we have
\begin{equation}
\mathcal U_k(\rho)=\left(u_1,\dots,u_{N_k(\rho)}\right).
\end{equation}

\subsection{Properties of constructed vectors}

The following propositions are noted, where the first proposition follows since $N_2(\rho)$ is even and the second follows since we are rotating lower-dimensional unit vectors to obtain the higher dimensional unit vectors.
\begin{proposition}
$N_k(\rho) $ is even $\forall \rho \in (0, \pi/2], k \ge 2$. \label{prop:M_even}
\end{proposition}

\begin{proposition}
For every \(k\ge 2\), every \(\rho\in(0,\pi/2]\), and every constructed vector \(u_i\in\mathcal U_k(\rho)\), we have $u_i\in\mathcal S^{k}$.
\end{proposition}

\subsection{Separation of the constructed vectors}

\begin{lemma}
Every pair of distinct constructed vectors is separated by angular distance at least \(\rho\). That is, $\forall u_i, u_j \in \mathcal U_k(\rho)$,
\begin{equation}
\theta_{\mathcal S}(u_i,u_j)\ge \rho
\qquad
\text{for all }i\neq j.
\end{equation}
\end{lemma}

\begin{proof}
For \(k=2\), the vectors are equally spaced around the circle with angular gap \(\gamma(\rho)\) between adjacent vectors. By \eqref{eqn:gamma-bound-new}, $\gamma(\rho)\ge \rho$. Therefore every pair of distinct vectors has angular distance at least \(\rho\).

Now assume the claim holds in dimension \(k-1\), and let \(k\ge 3\).

\noindent \textbf{Large-scale regime \(\pi/4 <\rho \le \pi/2\).} The two constructed vectors lie on the same meridian with polar angles \(\pi/4\) and \(3\pi/4\). Hence
\begin{equation}
\theta_{\mathcal S}(u_1,u_2)=\frac{\pi}{2}\ge \rho.
\end{equation}

\noindent \textbf{Small-scale regime \(0<\rho\le\pi/4\).} Let
\begin{equation}
u_{\ell,j} = \cos(\alpha_\ell)e_k+\sin(\alpha_\ell)\Lambda_k(v_j)
\end{equation}
and
\begin{equation}
u_{\ell',j'}
=
\cos(\alpha_{\ell'})e_k+\sin(\alpha_{\ell'})\Lambda_k(v_{j'})
\end{equation}
be two distinct constructed vectors.

Their inner product is
\begin{equation}
\label{eqn:inner-product-grid}
u_{\ell,j}^T u_{\ell',j'}
=
\cos(\alpha_\ell)\cos(\alpha_{\ell'})
+
\sin(\alpha_\ell)\sin(\alpha_{\ell'})
v_j^T v_{j'}.
\end{equation}

We next consider three exhaustive cases.

\noindent\textbf{Case 1: same longitude.}
Suppose \(j=j'\) and \(\ell\neq \ell'\). Since \(v_j^T v_j=1\), \eqref{eqn:inner-product-grid} gives
\begin{equation}
\theta_{\mathcal S}(u_{\ell,j},u_{\ell',j})
=
|\alpha_\ell-\alpha_{\ell'}|.
\end{equation}
Because $\alpha_\ell=\frac{\pi}{4}+(\ell-1)\rho$, we have $|\alpha_\ell-\alpha_{\ell'}| = |\ell-\ell'|\rho \ge \rho$. Hence, $
\theta_{\mathcal S}(u_{\ell,j},u_{\ell',j})\ge \rho$.

\noindent\textbf{Case 2: different longitudes and different polar levels.}
Suppose \(j\neq j'\) and \(\ell\neq \ell'\). Since \(v_j^T v_{j'}\le 1\), \eqref{eqn:inner-product-grid} gives
\begin{align}
u_{\ell,j}^T u_{\ell',j'}
&\le
\cos(\alpha_\ell)\cos(\alpha_{\ell'})
+
\sin(\alpha_\ell)\sin(\alpha_{\ell'})\\
&=
\cos(\alpha_\ell-\alpha_{\ell'}).
\end{align}
Therefore, $\theta_{\mathcal S}(u_{\ell,j},u_{\ell',j'}) \ge |\alpha_\ell-\alpha_{\ell'}| = |\ell-\ell'|\rho \ge \rho$.

\noindent\textbf{Case 3: different longitudes and the same polar level.}
Suppose \(j\neq j'\) and \(\ell=\ell'\). By the induction hypothesis of Theorem~\ref{thm:unit_vector_theorem} applied to
\begin{equation}
\mathcal U_{k-1}(2\rho)=\left(v_1,\dots,v_M\right),
\end{equation}
we have $\theta_{\mathcal S}(v_j,v_{j'})\ge 2\rho$ and $v_j^T v_{j'}\le \cos(2\rho)$. Hence
\begin{equation}
u_{\ell,j}^T u_{\ell,j'} = \cos^2(\alpha_\ell) + \sin^2(\alpha_\ell)v_j^T v_{j'} = 1-\sin^2(\alpha_\ell)(1-\cos(2\rho)).
\end{equation}
Since $\frac \pi4 \le \alpha_\ell \le \frac {3 \pi}4$ from Proposition~\ref{prop:alpha_bounds}, $\sin \alpha_\ell \ge \frac 1{\sqrt{2}}$. Thus, we get
\begin{equation}
u_{\ell,j}^T u_{\ell,j'} \le
1-\frac{1}{2}(1-\cos(2\rho)) =
\cos^2\rho.
\end{equation}
Since \(0<\rho\le\pi/4\), we have $\cos^2\rho\le \cos\rho$. Thus, $u_{\ell,j}^T u_{\ell,j'}\le \cos\rho$, which is equivalent to the desired statement.
\end{proof}

\subsection{Local transitions between consecutive vectors}

\begin{lemma}
\label{lem:local-transitions}
For every \(k\ge 2\), there exists a constant \(A_k>0\), depending only on \(k\), such that the ordered construction
$\mathcal U_k(\rho)$ satisfies
\begin{equation}
\theta_{\mathcal S}(u_i,u_{i+1})\le A_k\rho
\qquad \forall i=1,\dots,N_k(\rho)-1.
\end{equation}
Moreover, the cyclic local transition also holds:
\begin{equation}
\theta_{\mathcal S}(u_{N_k(\rho)},u_1)\le A_k\rho.
\end{equation}
\end{lemma}

\begin{proof}
We prove the claim by induction on \(k\).

For \(k=2\), consecutive vectors on the circle are separated by the angle $\gamma(\rho)=\frac{2\pi}{N_2(\rho)}$. The cyclic transition from \(u_{N_2(\rho)}\) back to \(u_1\) is also exactly \(\gamma(\rho)\). By \eqref{eqn:gamma-bound-new}, the claim holds for \(k=2\) with $A_2=2$.

Now assume the claim holds in dimension \(k-1\), and let \(k\ge 3\). We show that the claim holds with
\begin{equation}
A_k=2A_{k-1} = 2^{k-1}.
\end{equation}

\noindent \textbf{Large-scale regime.} First consider the large-scale regime \(\pi/4<\rho\le\pi/2\). The construction has exactly two vectors, and $\theta_{\mathcal S}(u_1,u_2)=\frac{\pi}{2}.$ Since \(\rho>\pi/4\), we have $\frac{\pi}{2}<2\rho$ and since \(A_k=2A_{k-1}\ge 2\), we have $2\rho \le A_k\rho$. From this, we have $\frac{\pi}{2} \le A_k \rho$ and thus the claim holds in the large-scale regime.

\noindent \textbf{Small-scale regime.} Recall from~\eqref{eqn:lower_dimensional_vectors_defn} that $\eta=2\rho$
and that
\begin{equation}
\mathcal U_{k-1}(\eta)=\left(v_1,\dots,v_M\right).
\end{equation}
By the induction hypothesis,
$\theta_{\mathcal S}(v_j,v_{j+1})\le A_{k-1}\eta \forall j=1,\dots,M-1$, and $
\theta_{\mathcal S}(v_M,v_1)\le A_{k-1}\eta$.

There are three kinds of local transitions in the longitude-first ordering.

\noindent\textbf{Case 1: Consecutive vectors on the same longitude.}
These vectors have the form
\begin{equation}
u_{\ell,j}
=
\cos(\alpha_\ell)e_k+\sin(\alpha_\ell)\Lambda_k(v_j)
\end{equation}
and
\begin{equation}
u_{\ell+1,j}
=
\cos(\alpha_{\ell+1})e_k+\sin(\alpha_{\ell+1})\Lambda_k(v_j),
\end{equation}
or the same pair in reverse order. Since the longitude vector \(v_j\) is the same,
\begin{equation}
\theta_{\mathcal S}(u_{\ell,j},u_{\ell+1,j})
=
\alpha_{\ell+1}-\alpha_\ell
=
\rho.
\end{equation}
Thus such angles are bounded by \(A_k\rho\).

\noindent\textbf{Case 2: Transition from one longitude to the next.}
Because the polar ordering alternates vector from one longitude to the next, the transition from longitude \(j\) to longitude \(j+1\) always occurs at a fixed polar level. More precisely, it is the pair $u_{L(\rho),j},\ u_{L(\rho),j+1} $ when \(j\) is odd, and the pair $u_{1,j},\ u_{1,j+1}$ when \(j\) is even.

Let
\begin{equation}
\beta_j=\theta_{\mathcal S}(v_j,v_{j+1}).
\end{equation}
By the induction hypothesis,
\begin{equation}
\beta_j\le A_{k-1}\eta.
\end{equation}

For either of the two transitions, write the common polar index as
\begin{equation}
\ell_j=
\begin{cases}
L(\rho), & j \text{ odd},\\
1, & j \text{ even}.
\end{cases}
\end{equation}
Then the two consecutive vectors are
\begin{equation}
u_{\ell_j,j}
=
\cos(\alpha_{\ell_j})e_k
+
\sin(\alpha_{\ell_j})\Lambda_k(v_j)
\end{equation}
and
\begin{equation}
u_{\ell_j,j+1}
=
\cos(\alpha_{\ell_j})e_k
+
\sin(\alpha_{\ell_j})\Lambda_k(v_{j+1}).
\end{equation}
Their inner product is
\begin{align}
u_{\ell_j,j}^{T}u_{\ell_j,j+1}
&= \cos^2(\alpha_{\ell_j}) + \sin^2(\alpha_{\ell_j})v_j^Tv_{j+1} \\
&= \cos^2(\alpha_{\ell_j}) + \sin^2(\alpha_{\ell_j})\cos\beta_j \\
&= \cos\beta_j + \cos^2(\alpha_{\ell_j})(1-\cos\beta_j) \\
&\ge \cos\beta_j.
\end{align}
Since \(\arccos\) is decreasing on \([-1,1]\), it follows that
\begin{equation}
\theta_{\mathcal S} \left(u_{\ell_j,j}, u_{\ell_j,j+1} \right) \le \beta_j.
\end{equation}
Therefore every transition from longitude \(j\) to longitude \(j+1\) satisfies
\begin{equation}
\theta_{\mathcal S} \left( u_{\ell_j,j}, u_{\ell_j,j+1}\right) \le A_{k-1}\eta = 2A_{k-1}\rho = A_k\rho.
\end{equation}

\noindent\textbf{Case 3: Cyclic transition from the last vector back to the first.}
Recall from Proposition~\ref{prop:M_even} that \(M\) is even. The last longitude \(j=M\) is traversed from \(\alpha_{L(\rho)}\) down to \(\alpha_1\). Hence the last vector of the whole ordered collection is
\begin{equation}
u_{1,M} = \cos(\alpha_1)e_k+\sin(\alpha_1)\Lambda_k(v_M),
\end{equation}
while the first vector is
\begin{equation}
u_{1,1} = \cos(\alpha_1)e_k+\sin(\alpha_1)\Lambda_k(v_1).
\end{equation}
The cyclic lower-dimensional transition gives
\begin{equation}
\theta_{\mathcal S}(v_M,v_1)\le A_{k-1}\eta.
\end{equation}
Repeating the fixed-polar-level estimate from Case 2 gives
\begin{equation}
\theta_{\mathcal S}(u_{1,M},u_{1,1})
\le A_{k-1}\eta = 2A_{k-1}\rho = A_k\rho.
\end{equation}

Combining the three cases proves both the consecutive local transition bound and the cyclic local transition bound.
\end{proof}

\subsection{Lower bound on the number of constructed vectors}

\begin{lemma}
There exists a constant \(c_k>0\), depending only on \(k\), such that
\begin{equation}
N_k(\rho)\ge c_k\rho^{-(k-1)}
\end{equation}
for every \(\rho\in(0,\pi/2]\).
\end{lemma}

\begin{proof}
We prove the claim by induction on \(k\), with the explicit choice
\begin{equation}
c_k =\frac{\pi^{k-1}}{2^{k^2}}.
\end{equation}

For \(k=2\), we have
$N_2(\rho)=2\left\lfloor \frac{\pi}{\rho}\right\rfloor.$
Since \(\rho\le\pi/2\), we have \(\pi/\rho\ge 2\), and therefore $
\left\lfloor \frac{\pi}{\rho}\right\rfloor \ge \frac{1}{2}\frac{\pi}{\rho}$. Thus,
$N_2(\rho)\ge \frac{\pi}{\rho}$. Since $c_2=\frac{\pi}{2^4}\le \pi$, the claim holds for \(k=2\).

Now assume the claim holds in dimension \(k-1\), and let \(k\ge 3\).

\noindent \textbf{Small-scale regime.} First suppose \(0<\rho\le\pi/4\). By \eqref{eqn:N-small-regime}, $N_k(\rho) = L(\rho)N_{k-1}(2\rho)$. Since
\begin{equation}
L(\rho) = \left\lfloor \frac{\pi}{2\rho}\right\rfloor+1
\ge
\frac{\pi}{2\rho},
\end{equation}
and since the induction hypothesis gives
\begin{equation}
N_{k-1}(2\rho)
\ge
c_{k-1}(2\rho)^{-(k-2)},
\end{equation}
we obtain
\begin{equation}
N_k(\rho) \ge \frac{\pi}{2\rho} \cdot c_{k-1}(2\rho)^{-(k-2)} = \frac{\pi c_{k-1}}{2^{k-1}} \rho^{-(k-1)} \ge c_k\rho^{-(k-1)}.
\end{equation}

\noindent \textbf{Large-scale regime.} Next suppose \(\pi/4<\rho\le\pi/2\). In this regime, $N_k(\rho)=2$. Also, $\rho^{-(k-1)} \le \left(\frac{4}{\pi}\right)^{k-1}$. Therefore
\begin{equation}
c_k\rho^{-(k-1)} \le \frac{\pi^{k-1}}{2^{k^2}} \left(\frac{4}{\pi}\right)^{k-1} = 2^{2k-2-k^2} \le 2.
\end{equation}
Hence
\begin{equation}
N_k(\rho)=2\ge c_k\rho^{-(k-1)}
\end{equation}

Combining the two regimes, the claim holds.
\end{proof}

Taking \(k=d\), \(A_d=2^{d-1}\), and
\(c_d=\frac{\pi^{d-1}}{2^{d^2}},\)
the constructed ordered collection \(\mathcal U_d(\rho)\) satisfies the desired conditions of the theorem.

\bibliographystyle{plainnat}
\bibliography{main}

\end{document}